\documentclass{article}

\usepackage{microtype}
\usepackage{graphicx}
\usepackage{booktabs} %

\usepackage{hyperref}

\usepackage{tikz,tkz-graph}
\usetikzlibrary{
  graphs,
  graphs.standard,
  positioning,chains,fit,shapes,calc
}

\newcommand{\name}{Graphite}

\usepackage{booktabs}       %
\usepackage{nicefrac}       %

\usepackage{amsmath,amsfonts,bbm,comment,mathtools,subcaption,amsthm}
\newtheorem{theorem}{Theorem}
\newtheorem{property}[theorem]{Property}

\usepackage[accepted]{icml2019}

\icmltitlerunning{Graphite: Iterative Generative Modeling of Graphs}

\begin{document}

\twocolumn[
\icmltitle{Graphite: Iterative Generative Modeling of Graphs}

\icmlsetsymbol{equal}{*}

\begin{icmlauthorlist}
\icmlauthor{Aditya Grover}{sta}
\icmlauthor{Aaron Zweig}{sta}
\icmlauthor{Stefano Ermon}{sta}
\end{icmlauthorlist}

\icmlaffiliation{sta}{Department of Computer Science, Stanford University, USA}

\icmlcorrespondingauthor{Aditya Grover}{adityag@cs.stanford.edu}

\icmlkeywords{Machine Learning, Unsupervised Learning, Graphs, Generative Modeling, Deep Learning}

\vskip 0.3in
]
\printAffiliationsAndNotice{}

\begin{abstract}
  Graphs are a fundamental abstraction for modeling relational data. 
  However, graphs are discrete and combinatorial in nature, and learning representations suitable for machine learning tasks poses statistical and computational challenges. 
  In this work, we propose \textit{\name{}}, an algorithmic framework for unsupervised learning of representations over nodes in large graphs using deep latent variable generative models. Our model parameterizes variational autoencoders (VAE) with graph neural networks, and uses a novel iterative graph refinement strategy inspired by low-rank approximations for decoding.
On a wide variety of synthetic and benchmark datasets, \name{} outperforms competing approaches for the tasks of density estimation, link prediction, and node classification. 
Finally, we derive a theoretical connection between message passing in graph neural networks and mean-field variational inference.
\end{abstract}

\section{Introduction}

Latent variable generative modeling is an effective approach for unsupervised representation learning of high-dimensional data~\citep{loehlin1998latent}. In recent years, representations learned by latent variable models parameterized by deep neural networks have shown impressive performance on many tasks such as semi-supervised learning and structured prediction~\citep{kingma2014semi,sohn2015learning}. However, these successes have been largely restricted to specific data modalities such as images and speech. In particular, it is challenging to apply current deep generative models for large scale graph-structured data which arise in a wide variety of domains in physical sciences, information sciences, and social sciences. 

To effectively model the relational structure of large graphs for deep learning, prior works have proposed to use \textit{graph neural networks}~\citep{gori2005new,scarselli2009graph,bruna2013spectral}. A graph neural network learns node-level representations by parameterizing an iterative message passing procedure between nodes and their neighbors.  
The tasks which have benefited from graph neural networks, including semi-supervised learning~\citep{kipf2016semi} and few shot learning~\citep{garcia2017few}, involve \textit{encoding} an input graph to a final output representation (such as the labels associated with the nodes). The inverse problem of learning to \textit{decode} a hidden representation into a graph, as in the case of a latent variable generative model, is a pressing challenge
that we address in this work. 

We propose \textit{\name{}}, a latent variable generative model 
for graphs based on variational autoencoding~\citep{kingma2013auto}. Specifically, we learn a directed model expressing a joint distribution over the entries of adjacency matrix of graphs 
and latent feature vectors for every node.
Our framework uses graph neural networks for inference (encoding) and generation (decoding).  
While the encoding is straightforward and can use any existing graph neural network, the decoding of these latent features to reconstruct the original graph is done using a multi-layer iterative procedure. The procedure starts with an initial reconstruction based on the inferred latent features, and iteratively refines the reconstructed graph via a message passing operation. 
The iterative refinement can be efficiently implemented using graph neural networks. 
In addition to the \name{} model, we also contribute to the theoretical understanding of graph neural networks by deriving equivalences between message passing in graph neural networks with mean-field inference in latent variable models via kernel embeddings~\citep{smola2007hilbert,dai2016discriminative}, formalizing what has thus far has been largely speculated empirically to the best of our knowledge~\citep{yoon2018inference}.

In contrast to recent works focussing on generation of small graphs \textit{e.g.}, molecules
\citep{you2018graphrnn,li2018learning}, the \name{} framework is particularly suited for representation learning on large graphs. Such representations are useful for several downstream tasks. In particular, we demonstrate that representations learned via \name{} outperform competing approaches for graph representation learning empirically for the tasks of density estimation (over entire graphs), link prediction, and semi-supervised node classification on synthetic and benchmark datasets.

\section{Preliminaries}
Throughout this work, we assume that all probability distributions admit absolutely continuous densities on a suitable reference measure. 
Consider a weighted undirected graph $G=(V, E)$ where $V$ and $E$ denote index sets of nodes and edges respectively.
Additionally, we denote the (optional) feature matrix associated with the graph as $\mathbf{X}\in \mathbb{R}^{n \times m}$ for an $m$-dimensional signal associated with each node, for \textit{e.g.}, these could refer to the user attributes in a social network.
We represent the graph structure using a symmetric adjacency matrix $\mathbf{A} \in \mathbb{R}^{n \times n} $ where $n=\vert V \vert$ and the entries $A_{ij}$ denote the weight of the edge between node $i$ and $j$.

\subsection{Weisfeiler-Lehman algorithm}
The 
Weisfeiler-Lehman (WL) algorithm~\citep{weisfeiler1968reduction,douglas2011weisfeiler} is a heuristic test of graph isomorphism between any two graphs $G$ and $G'$. The algorithm proceeds in iterations. 
Before the first iteration, we label every node in $G$ and $G'$ with a scalar \textit{isomorphism invariant} initialization (\textit{e.g.}, node degrees). That is, if $G$ and $G'$ are assumed to be isomorphic, then an isomorphism invariant initialization is one where the matching nodes establishing the isomorphism in $G$ and $G'$ have the same labels (a.k.a. messages).
Let $\mathbf{H}^{(0)}=[h_1^{(l)}, h_2^{(l)}, \cdots, h_n^{(l)}]^T$ denote the vector of initializations 
for the nodes in the graph at iteration $l\in \mathbb{N} \cup 0$.
 At every iteration $l>0$, we perform a relabelling of nodes in $G$ and $G'$ based on a message passing update rule:
\begin{align}\label{eq:wl_mp}
\mathbf{H}^{(l)} \leftarrow \mathrm{hash}\left(\mathbf{A}\mathbf{H}^{(l-1)} \right)
\end{align}
where $\mathbf{A}$ denotes the adjacency matrix of the corresponding graph and $\mathrm{hash}:\mathbb{R}^n\to\mathbb{R}^n$ is any suitable hash function \textit{e.g.}, a non-linear activation. Hence, the message 
for every node is computed as a hashed sum of the messages from the neighboring nodes (since $A_{ij}\neq 0$ only if $i$ and $j$ are neighbors). We repeat the process for a specified number of iterations, or until convergence. If the label sets for the nodes in $G$ and $G'$ are equal (which can be checked using sorting in $O(n\log n)$ time), then the algorithm declares the two graphs $G$ and $G'$ to be isomorphic. 

The ``$k$-dim" WL algorithm extends the 1-dim algorithm above by simultaneously passing messages of length $k$ (each initialized with some isomorphism invariant scheme).
A positive test for isomorphism requires equality in all $k$ dimensions for nodes in $G$ and $G'$ after the termination of message passing. 
This algorithmic test is a heuristic which guarantees no false negatives but can give false positives wherein two non-isomorphic graphs can be falsely declared isomorphic. 
Empirically, the test has been shown to fail on some regular graphs but gives excellent performance on real-world graphs~\citep{shervashidze2011weisfeiler}.

\subsection{Graph neural networks}
Intuitively, the WL algorithm encodes the structure of the graph in the form of messages at every node.
Graph neural networks (GNN) build on this observation and parameterize an unfolding of the iterative message passing procedure which we describe next.

A GNN consists of many layers, indexed by $l\in \mathbb{N}$ with each layer associated with an activation $\eta_l$ and a dimensionality $d_l$. In addition to the input graph $\mathbf{A}$, every layer $l\in \mathbb{N}$ of the GNN takes as input the activations from the previous layer $\mathbf{H}^{(l-1)} \in \mathbb{R}^{n \times d_{l-1}}$, a family of linear transformations $\mathcal{F}_{l}: \mathbb{R}^{n \times n} \rightarrow \mathbb{R}^{n \times n}$, and a matrix of learnable weight parameters $\mathbf{W}_{l} \in \mathbb{R}^{d_{l-1} \times d_{l}}$ and optional bias parameters $\mathbf{B}_{l} \in \mathbb{R}^{n \times d_{l}}$. Recursively, the layer wise propagation rule in a GNN is given by:
\begin{align}\label{eq:gnn_mp}
\mathbf{H}^{(l)} \leftarrow \eta_{l}\left( \mathbf{B}_l + \sum_{f \in \mathcal{F}_{l}}f(\mathbf{A})\mathbf{H}^{(l-1)} \mathbf{W}_{l} \right)
\end{align}
with the base cases $\mathbf{H}^{(0)} = \mathbf{X}$ and $d_0=m$.
Here, $m$ is the feature dimensionality.
If there are no explicit node features, we set $\mathbf{H}^{(0)} = \mathbf{I}_n$ (identity) and $d_0=n$.  
Several variants of graph neural networks have been proposed in prior work. For instance, graph convolutional networks (GCN) \cite{kipf2016semi} instantiate graph neural networks with a permutation equivariant propagation rule:
\begin{align}\label{eq:gcn_mp}
\mathbf{H}^{(l)} \leftarrow \eta_{l}\left(\mathbf{B}_l + \tilde{\mathbf{A}}\mathbf{H}^{(l-1)} \mathbf{W}_{l}\right)
\end{align}
where $\tilde{\mathbf{A}}=\mathbf{D}^{-1/2}\mathbf{A}\mathbf{D}^{-1/2}$ is the symmetric diagonalization of $\mathbf{A}$ given the diagonal degree matrix $\mathbf{D}$ (\textit{i.e.}, $D_{ii}= \sum_{(i,j) \in E} A_{ij}$), and same base cases as before.
Comparing the above with the WL update rule in Eq.~\eqref{eq:wl_mp}, we can see that the activations for every layer in a GCN are computed via parameterized, scaled activations (messages) of the previous layer being propagated over the graph, with the hash function implicitly specified using an activation function $\eta_{l}$. 

Our framework is agnostic to  instantiations of message passing rule of a graph neural network in Eq.~\eqref{eq:gnn_mp}, and we use graph convolutional networks for experimental validation due to the permutation equivariance property. For brevity, we denote the output $\mathbf{H}$ for the final layer of a multi-layer graph neural network with input adjacency matrix $\mathbf{A}$, node feature matrix $\mathbf{X}$, and parameters $\langle\mathbf{W},\mathbf{B}\rangle$ as $\mathbf{H} = \mathrm{GNN}_{\langle\mathbf{W},\mathbf{B}\rangle}(\mathbf{A}, \mathbf{X})$, with appropriate activation functions and linear transformations applied at each hidden layer of the network.

\section{Generative Modeling via \name{}}\label{sec:framework}

For generative modeling of graphs, we are interested in learning a parameterized distribution over adjacency matrices $\mathbf{A}$. In this work, we restrict ourselves to modeling graph \textit{structure} only, and any additional information in the form of node features $\mathbf{X}$ is incorporated as conditioning evidence.

\begin{figure}[t]
\centering
\includegraphics[width=0.45\textwidth]{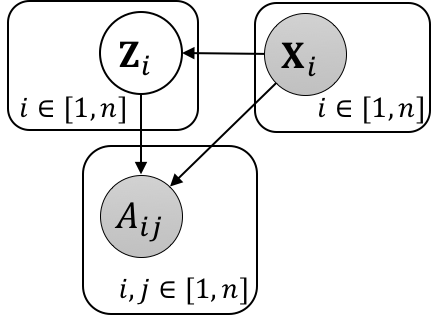}
\caption{Latent variable model for \name{}. Observed evidence variables in gray.}\label{fig:lvm}
\end{figure}
In \name{}, we adopt a latent variable approach for modeling the generative process. That is, we introduce latent variable vectors $\mathbf{Z}_i \in \mathbb{R}^{k}$ and evidence feature vectors 
$\mathbf{X}_i \in \mathbb{R}^{m}$ for each node $i \in \{1,2,\cdots,n\}$ along with an observed variable for each pair of nodes $A_{ij} \in \mathbb{R}$. Unless necessary, we use a succinct representation $\mathbf{Z} \in \mathbb{R}^{n \times k}$, $\mathbf{X} \in \mathbb{R}^{n \times m}$, and $\mathbf{A} \in \mathbb{R}^{n \times n}$ for the variables henceforth. 
The conditional independencies between the variables can be summarized in the directed graphical model (using plate notation) in Figure~\ref{fig:lvm}. We can learn the model parameters $\theta$ by maximizing the marginal likelihood of the observed adjacency matrix conditioned on $\mathbf{X}$:
\begin{align}\label{eq:marginal_ll}
\max_\theta \log p_\theta (\mathbf{A} \vert \mathbf{X}) = \log \int_{\mathbf{Z}} p_\theta (\mathbf{A}, \mathbf{Z} \vert \mathbf{X}) \mathrm{d}\mathbf{Z}
\end{align}

Here, $p(\mathbf{Z} \vert \mathbf{X})$ is a fixed prior distribution over the latent features of every node \textit{e.g.}, isotropic Gaussian.
If we have multiple graphs in our dataset, we maximize the expected log-likelihoods over all the corresponding adjacency matrices. We can obtain a tractable, stochastic evidence lower bound (ELBO) to the above objective by
introducing a variational posterior $q_\phi(\mathbf{Z}  \vert \mathbf{A}, \mathbf{X})$ with parameters $\phi$: 
\begin{align}\label{eq:elbo}
\log p_\theta (\mathbf{A} \vert \mathbf{X}) \geq \mathbb{E}_{q_\phi(\mathbf{Z}\vert \mathbf{A}, \mathbf{X})}\left [\log \frac{p_\theta(\mathbf{A}, \mathbf{Z} \vert \mathbf{X})}{q_\phi(\mathbf{Z}\vert \mathbf{A}, \mathbf{X})}\right]
\end{align}

The lower bound is tight when the variational posterior $q_\phi(\mathbf{Z}\vert \mathbf{A}, \mathbf{X})$ matches the true posterior $p_\theta(\mathbf{Z} \vert \mathbf{A}, \mathbf{X})$ and hence maximizing the above objective optimizes for the parameters that define the best approximation to the true posterior within the variational family~\citep{blei2017variational}.
We now discuss parameterizations for specifying  $q_\phi(\mathbf{Z}\vert \mathbf{A}, \mathbf{X})$ (\textit{i.e.}, encoder) and $p_\theta(\mathbf{A} \vert \mathbf{Z}, \mathbf{X})$ (\textit{i.e.}, decoder).

\paragraph{Encoding using forward message passing.}

Typically we use the mean field approximation for defining the variational family and hence:
\begin{align}\label{eq:enc_mf}
q_\phi(\mathbf{Z} \vert \mathbf{A}, \mathbf{X}) \approx \prod_{i=1}^n q_{\phi_i}(\mathbf{Z}_i \vert \mathbf{A}, \mathbf{X})
\end{align}
Additionally, we would like to make distributional assumptions on each variational marginal density $q_{\phi_i}(\mathbf{Z}_i\vert \mathbf{A}, \mathbf{X})$ such that it is reparameterizable and easy-to-sample, such that the gradients w.r.t. $\phi_i$ have low variance~\citep{kingma2013auto}. In \name{}, we assume isotropic Gaussian variational marginals with diagonal covariance.
The parameters for the variational marginals $q_{\phi_i}(\mathbf{Z}\vert \mathbf{A}, \mathbf{X})$ are specified using a graph neural network: 
\begin{align}\label{eq:graphite_enc}
\boldsymbol{\mu}, \boldsymbol{\sigma} &= \mathrm{GNN}_{\phi}(\mathbf{A}, \mathbf{X})
\end{align}
where $\boldsymbol{\mu}$ and $\boldsymbol{\sigma}$ denote the vector of means and standard deviations for the variational marginals $\{q_{\phi_i}(\mathbf{Z}_i\vert \mathbf{A}, \mathbf{X})\}_{i=1}^n$ and $\phi=\{\phi_i\}_{i=1}^n$ are the full set of variational parameters.

\paragraph{Decoding using reverse message passing.} For specifying the observation model $p_\theta(\mathbf{A} \vert \mathbf{Z}, \mathbf{X})$, we cannot directly use a graph neural network since we do not have an input graph for message passing. To sidestep this issue, we propose an iterative two-step approach that alternates between defining an intermediate graph and then gradually refining this graph through message passing.
Formally, given a latent matrix $\mathbf{Z}$ and an input feature matrix $\mathbf{X}$, we iterate over the following sequence of operations:
\begin{align}
\widehat{\mathbf{A}} &= \frac{\mathbf{Z} \mathbf{Z}^T}{\lVert \mathbf{Z} \rVert^2} + \mathbf{1}\mathbf{1}^T,\label{eq:int_graph_decoding}\\
\mathbf{Z}^\ast &= \mathrm{GNN}_\theta(\widehat{\mathbf{A}},  [\mathbf{Z} \vert \mathbf{X}]) \label{eq:refine_graph_dec}
\end{align}
where the second argument to the GNN is a concatenation of  $\mathbf{Z}$ and $\mathbf{X}$. The first step constructs an intermediate weighted graph $\widehat{\mathbf{A}} \in \mathbb{R}^{n \times n}$ by applying an inner product of $\mathbf{Z}$ with itself and adding an additional constant of 1 to ensure entries are non-negative. 
And the second step performs a pass through a parameterized graph neural network. We can repeat the above sequence to gradually refine the feature matrix $\mathbf{Z}^\ast$. The final distribution over graph parameters is obtained using an inner product step on $\mathbf{Z}^\ast\in \mathbb{R}^{n \times k^\ast}$ akin to Eq.~\eqref{eq:int_graph_decoding}, where $k^\ast \in \mathbb{N}$ is determined via the network architecture.
For efficient sampling, we assume the observation model factorizes:
\begin{align}\label{eq:dec_mf}
p_\theta(\mathbf{A} \vert \mathbf{Z}, \mathbf{X}) = \prod_{i=1}^n \prod_{j=1}^n p_\theta^{(i,j)}(\mathbf{A}_{ij} \vert \mathbf{Z}^\ast).
\end{align}
The distribution over the individual edges can be expressed as a Bernoulli or Gaussian distribution for unweighted and real-valued edges respectively. E.g., the edge probabilities for an unweighted graph are given as $\mathrm{sigmoid}(\mathbf{Z}^\ast \mathbf{Z}^{\ast^T})$.

\begin{table*}[t]
  \caption{Mean reconstruction errors and negative log-likelihood estimates (in nats) for autoencoders and variational autoencoders respectively on test instances from six different generative families. Lower is better.
  }
    \vspace{0.05in}
  \label{table-elbo}
  \centering
  \resizebox{\textwidth}{!}{\begin{tabular}{|c|c|c|c|c|c|c|}
    \toprule
	& Erdos-Renyi & Ego & Regular & Geometric & Power Law & Barabasi-Albert \\
    \midrule
    GAE & 221.79 $\pm$ 7.58 & 197.3 $\pm$ 1.99 & 198.5 $\pm$ 4.78 & 514.26 $\pm$ 41.58 & 519.44 $\pm$ 36.30 & 236.29 $\pm$ 15.13\\
        \name{}-AE & \textbf{195.56} $\pm$ 1.49 & \textbf{182.79} $\pm$ 1.45 & \textbf{191.41} $\pm$ 1.99 & \textbf{181.14} $\pm$ 4.48 & \textbf{201.22} $\pm$ 2.42 & \textbf{192.38} $\pm$ 1.61\\
    \midrule
    VGAE & 273.82 $\pm$ 0.07 & 273.76 $\pm$ 0.06 & 275.29 $\pm$ 0.08 & 274.09 $\pm$ 0.06 & 278.86 $\pm$ 0.12 & 274.4 $\pm$ 0.08\\
    \name{}-VAE & \textbf{270.22} $\pm$ 0.15 & \textbf{270.70} $\pm$ 0.32 & \textbf{266.54} $\pm$ 0.12 & \textbf{269.71} $\pm$ 0.08 & \textbf{263.92} $\pm$ 0.14 & \textbf{268.73} $\pm$ 0.09\\

    \bottomrule
  \end{tabular}}
\end{table*}

\subsection{Scalable learning \& inference in \name{}} 

For representation learning of large graphs, we require the encoding and decoding steps  in \name{} to be computationally efficient. On the surface, the decoding step involves inner products of potentially dense matrices $\mathbf{Z}$, which is an $O(n^2k)$ operation. 
Here, $k$ is the dimension of the per-node latent vectors $\mathbf{Z}_i$ used to define $\widehat{\mathbf{A}}$. 

For any intermediate decoding step as in Eq.~\eqref{eq:int_graph_decoding}, we propose to offset this expensive computation by using the associativity property of matrix multiplications for the message passing step in Eq.~\eqref{eq:refine_graph_dec}. For notational brevity, consider the simplified graph propagation rule for a GNN:
\begin{align*}
\mathbf{H}^{(l)} \leftarrow \eta_{l}\left( \widehat{\mathbf{A}}\mathbf{H}^{(l-1)}\right)
\end{align*}
where $\widehat{\mathbf{A}}$ is defined in Eq.~\eqref{eq:int_graph_decoding}. 

Instead of directly taking an inner product of $\mathbf{Z}$ with itself, we note that the subsequent operation involves another matrix multiplication and hence, we can perform right multiplication instead.
If $d_l$ and $d_{l-1}$ denote the size of the layers $\mathbf{H}^{(l)}$ and $\mathbf{H}^{(l-1)}$ respectively, then the time complexity of propagation based on right multiplication is given by $O(nkd_{l-1} + nd_{l-1}d_l)$. 

The above trick sidesteps the quadratic $n^2$ complexity for decoding in the intermediate layers without any loss in statistical accuracy. The final layer however still involves an inner product with respect to $\mathbf{Z}^\ast$ between potentially dense matrices. However, since the edges are generated independently, we can approximate the loss objective by performing a Monte Carlo evaluation of the reconstructed adjacency matrix parameters in Eq.~\eqref{eq:dec_mf}. By adaptively choosing the number of entries for Monte Carlo approximation, we can trade-off statistical accuracy for computational budget. 

\section{Experimental Evaluation}\label{sec:exps}

We evaluate \name{} on tasks involving entire graphs, nodes, and edges. 
We consider two variants of our proposed framework: the \textit{\name{}-VAE}, which corresponds to a directed latent variable model as described in Section~\ref{sec:framework} and \textit{\name{}-AE}, which corresponds to an autoencoder trained to minimize the error in reconstructing an input adjacency matrix. For unweighted graphs (\textit{i.e.}, $\mathbf{A}\in \{0,1\}^{n \times n}$), the reconstruction terms in the objectives for both \name{}-VAE and \name{}-AE minimize the negative cross entropy between the input and reconstructed adjacency matrices. For weighted graphs, we use the mean squared error. Additional hyperparameter details are described in Appendix~\ref{app:expt}.

\subsection{Reconstruction \& density estimation}

In the first set of tasks, we evaluate learning in \name{} based on held-out reconstruction losses and log-likelihoods estimated by the learned \name{}-VAE and \name{}-AE models respectively on a collection of graphs with varying sizes. In direct contrast to modalities such as images, graphs cannot be straightforwardly reduced to a fixed number of vertices for input to a graph convolutional network. One simplifying modification taken by \citet{bojchevski2018netgan} is to consider only the largest connected component for evaluating and optimizing the objective, which we appeal to as well.  Thus by setting the dimensions of $\mathbf{Z}^*$ to a maximum number of vertices, \name{} can be used for inference tasks over entire graphs with potentially smaller sizes by considering only the largest connected component.

We create datasets from six graph families with fixed, known generative processes: the Erdos-Renyi, ego-nets, random regular graphs,  random geometric graphs, random Power Law Tree and Barabasi-Albert. For each family, 300 graph instances were sampled with each instance having $10-20$ nodes and evenly split into train/validation/test instances. As a benchmark comparison, we compare against the Graph Autoencoder/Variational Graph Autoencoder (GAE/VGAE)~\citep{kipf2016variational}. The GAE/VGAE models consist of an encoding procedure similar to \name{}. However, the decoder has no learnable parameters and reconstruction is done solely through an inner product operation (such as the one in Eq.~\eqref{eq:int_graph_decoding}). 

The mean reconstruction errors and the negative log-likelihood results on a test set of instances are shown in Table~\ref{table-elbo}. Both \name{}-AE and \name{}-VAE outperform AE and VGAE significantly on these tasks, indicating the usefulness of learned decoders in \name{}.

\begin{table}[t]
\centering
  \caption{Citation network statistics}
  \label{table-stats}
  \vspace{0.05in}
  \centering
  \begin{tabular}{|c|c|c|c|c|}
    \toprule
	& Nodes & Edges & Node Features & Labels \\
    \midrule
    Cora & 2708 & 5429 & 1433 & 7\\
    Citeseer & 3327 & 4732 & 3703 & 6\\
    Pubmed & 19717 & 44338 & 500 & 3\\
    \bottomrule
  \end{tabular}
\end{table}

\begin{table*}[t]
  \caption{Area Under the ROC Curve (AUC) for link prediction (* denotes dataset with features). Higher is better.
  }
  \label{table-auc}
   \vspace{0.05in}
  \centering
  \begin{tabular}{|c|c|c|c|c|c|c|}
    \toprule
	& Cora &  Citeseer & Pubmed & Cora* & Citeseer* & Pubmed* \\
    \midrule
    SC & 89.9 $\pm$ 0.20 & 91.5 $\pm$ 0.17& \textbf{94.9} $\pm$ 0.04 & - & - & -\\
    DeepWalk & 85.0 $\pm$ 0.17& 88.6 $\pm$ 0.15& 91.5 $\pm$ 0.04& - & - & -\\
    node2vec & 85.6 $\pm$ 0.15& 89.4 $\pm$ 0.14& 91.9 $\pm$ 0.04 & - & - & -\\
    GAE & 90.2 $\pm$ 0.16& 92.0 $\pm$ 0.14& 92.5 $\pm$ 0.06& 93.9 $\pm$ 0.11& 94.9 $\pm$ 0.13& 96.8 $\pm$ 0.04\\
    VGAE & 90.1 $\pm$ 0.15& 92.0 $\pm$ 0.17& 92.3 $\pm$ 0.06 & 94.1 $\pm$ 0.11& 96.7 $\pm$ 0.08& 95.5 $\pm$ 0.13\\
    \midrule
    \name{}-AE & 91.0 $\pm$ 0.15 & 92.6 $\pm$ 0.16& 94.5 $\pm$ 0.05& 94.2 $\pm$ 0.13& 96.2 $\pm$ 0.10& \textbf{97.8} $\pm$ 0.03 \\
    \name-VAE & \textbf{91.5} $\pm$ 0.15 & \textbf{93.5} $\pm$ 0.13 & 94.6 $\pm$ 0.04& \textbf{94.7} $\pm$ 0.11 & \textbf{97.3} $\pm$ 0.06 & 97.4 $\pm$ 0.04\\
    \bottomrule
  \end{tabular}
\end{table*}

\normalsize
\begin{table*}[t]
  \caption{Average Precision (AP) scores for link prediction (* denotes dataset with features). Higher is better.}
  \label{table-ap}
  \centering
   \vspace{0.05in}
  \begin{tabular}{|c|c|c|c|c|c|c|}
    \toprule
	& Cora & Citeseer & Pubmed & Cora* & Citeseer* & Pubmed* \\
    \midrule
    SC & 92.8 $\pm$ 0.12 & 94.4 $\pm$ 0.11& \textbf{96.0} $\pm$ 0.03 & - & - & -\\
    DeepWalk & 86.6 $\pm$ 0.17& 90.3 $\pm$ 0.12& 91.9 $\pm$ 0.05& - & - & -\\
    node2vec & 87.5 $\pm$ 0.14& 91.3 $\pm$ 0.13& 92.3 $\pm$ 0.05 & - & - & -\\
    GAE & 92.4 $\pm$ 0.12& 94.0 $\pm$ 0.12& 94.3 $\pm$ 0.5& 94.3 $\pm$ 0.12& 94.8 $\pm$ 0.15& 96.8 $\pm$ 0.04\\
    VGAE & 92.3 $\pm$ 0.12& 94.2 $\pm$ 0.12& 94.2 $\pm$ 0.04& 94.6 $\pm$ 0.11& 97.0 $\pm$ 0.08& 95.5 $\pm$ 0.12\\
    \midrule
\name-AE & 92.8 $\pm$ 0.13& 94.1 $\pm$ 0.14& 95.7 $\pm$ 0.06& 94.5 $\pm$ 0.14& 96.1 $\pm$ 0.12& \textbf{97.7} $\pm$ 0.03 \\
\name-VAE & \textbf{93.2} $\pm$ 0.13 & \textbf{95.0} $\pm$ 0.10 & \textbf{96.0} $\pm$ 0.03 & \textbf{94.9} $\pm$ 0.13 & \textbf{97.4} $\pm$ 0.06 & 97.4 $\pm$ 0.04\\

    \bottomrule
  \end{tabular}
\end{table*}

\normalsize

\begin{figure}[ht]
\centering
\begin{subfigure}[b]{0.45\textwidth}
\centering
\includegraphics[width=\textwidth]{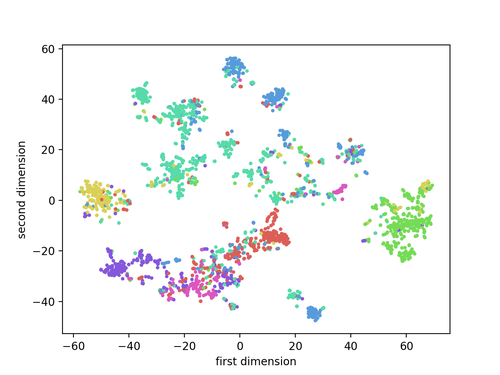}
\caption{Graphite-AE}
\end{subfigure}
\begin{subfigure}[b]{0.45\textwidth}
\centering
\includegraphics[width=\textwidth]{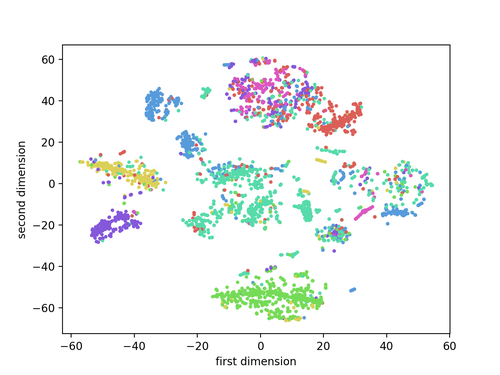}
\caption{Graphite-VAE}
\end{subfigure}
\caption{t-SNE embeddings of the latent feature vectors for the Cora dataset. Colors denote labels.}\label{fig:tsne_clustering} 
\vspace{-0.05in}
\end{figure}

\subsection{Link prediction}

The task of link prediction is to predict whether an edge exists between a pair of nodes~\citep{loehlin1998latent}. Even though \name{} learns a distribution over graphs, it can be used for predictive tasks within a \textit{single} graph. In order to do so, we learn a model for a random, connected training subgraph of the true graph. For validation and testing, we add a balanced set of positive and negative (false) edges to the original graph and evaluate the model performance based on the reconstruction probabilities assigned to the validation and test edges (similar to \textit{denoising} of the input graph). In our experiments, we held out a set of $5\%$ edges for validation, $10\%$ edges for testing, and train all models on the remaining subgraph. Additionally, the validation and testing sets also each contain an equal number of non-edges.

 \paragraph{Datasets.} We compared across standard benchmark citation network datasets: Cora, Citeseer, and Pubmed with papers as nodes and citations as edges~\citep{sen2008networks}. 
 The node-level features correspond to the text attributes in the papers.
The dataset statistics are summarized in Table~\ref{table-stats}.

\paragraph{Baselines and evaluation metrics.} We evaluate performance based on the Area Under the ROC Curve (AUC) and Average Precision (AP) metrics. We evaluated \name{}-VAE and \name{}-AE against the following baselines: Spectral Clustering (SC)~\citep{tang2011leveraging}, DeepWalk~\citep{perozzi2014deepwalk}, node2vec~\citep{grover2016node2vec}, and GAE/VGAE~\citep{kipf2016variational}. SC, DeepWalk, and node2vec do not provide the ability to incorporate node features while learning embeddings, and hence we evaluate them only on the featureless datasets.

\paragraph{Results.} The AUC and AP results (along with standard errors) are shown in Table~\ref{table-auc} and Table~\ref{table-ap} respectively averaged over 50 random train/validation/test splits. 
On both metrics, \name{}-VAE gives the best performance overall. \name{}-AE also gives good results, generally outperforming its closest competitor GAE.

\paragraph{Qualitative evaluation.} We  visualize the embeddings learned by \name{} and given by a 2D t-SNE projection~\citep{maaten2008visualizing} of the latent feature vectors (given as rows for $\mathbf{Z}$ with $\lambda=0.5$) on the Cora dataset in Figure~\ref{fig:tsne_clustering}. Even without any access to label information for the nodes during training, the name{} models are able to cluster the nodes (papers) as per their labels (paper categories).

\subsection{Semi-supervised node classification}

Given labels for a subset of nodes in an underlying graph, the goal of this task is to predict the labels for the remaining nodes. We consider a \textit{transductive} setting, where we have access to the test nodes (without their labels) during training.  

Closest approach to \name{} for this task is a supervised graph convolutional network (GCN) trained end-to-end. We consider an extension of this baseline, wherein we augment the GCN objective with the \name{} objective and a hyperparameter to control the relative importance of the two terms in the combined objective. The parameters $\phi$ for the encoder are shared across these two objectives, with an additional GCN layer for mapping the encoder output to softmax probabilities over the requisite number of classes. All parameters are learned jointly.

\paragraph{Results.} The classification accuracy of the semi-supervised models is given in Table~\ref{table-acc}. We find that \name{}-hybrid outperforms the competing models on all datasets and in particular the GCN approach which is the closest baseline. Recent work in Graph Attention Networks shows that extending GCN by incoporating attention can boost performance on this task~\citep{velickovic2018graph}. 
Using GATs in place of GCNs for parameterizing \name{} could yield similar performance boost in future work.

\begin{table}[t]
\centering
  \caption{Classification accuracies (* denotes dataset with features). Baseline numbers from~\citet{kipf2016semi}.}
  \label{table-acc}
  \centering
\vspace{0.05in}
  \begin{tabular}{|c|c|c|c|}
    \toprule
	& Cora* & Citeseer* & Pubmed* \\
    \midrule
    SemiEmb & 59.0 & 59.6 & 71.1 \\
    DeepWalk & 67.2 & 43.2 & 65.3 \\
    ICA & 75.1 & 69.1 & 73.9 \\
    Planetoid & 75.7 & 64.7 & 77.2 \\
    GCN & 81.5 & 70.3 & 79.0\\ \midrule
    Graphite & \textbf{82.1} $\pm$ 0.06 & \textbf{71.0} $\pm$ 0.07 & \textbf{79.3} $\pm$ 0.03\\
    \bottomrule
  \end{tabular}
\end{table}

\section{ Theoretical Analysis}\label{sec:interpret}

In this section, we derive a theoretical connection between message passing in graph neural networks and approximate inference in related undirected graphical models. 

\subsection{Kernel embeddings}
We first provide a brief background on kernel embeddings.
A kernel defines a notion of similarity between pairs of objects~\citep{scholkopf2002learning,shawe2004kernel}.
Let $K: \mathcal{Z} \times \mathcal{Z} \rightarrow \mathbb{R}$ be the kernel function defined over a space of objects, say $\mathcal{Z}$. 
With every kernel function $K$, we have an associated feature map $\psi:\mathcal{Z} \to \mathcal{H}$ where $\mathcal{H}$ is a potentially infinite dimensional feature space.

Kernel methods can be used to specify embeddings of \textit{distributions} of arbitrary objects~\citep{smola2007hilbert,gretton2007kernel}. 
Formally, we denote these functional mappings as $T_{\psi}: \mathcal{P} \rightarrow \mathcal{H}$ where $\mathcal{P}$ specifies the space of all distributions on $\mathcal{Z}$. 
These mappings, referred to as kernel embeddings of distributions, are defined as:
\begin{align}
T_{\psi}(p) := \mathbb{E}_{Z\sim p}[\psi(Z)]
\end{align}
for any $p \in \mathcal{P}$.
We are particularly interested in kernels with feature maps $\psi$ that define injective embeddings, \textit{i.e.}, for any pair of distributions $p_1$ and $p_2$, we have $T_{\psi}(p_1)\neq T_{\psi}(p_2)$ if $p_1 \neq p_2$. 
For injective embeddings, 
we can compute functionals of any distribution by directly applying a corresponding function on its embedding. Formally, for every function $\mathcal{O}:\mathcal{P}\rightarrow \mathbb{R}^d$, $ d\in \mathbb{N}$ and injective embedding $T_{\psi}$, there exists a function $\tilde{\mathcal{O}}_{\psi}:\mathcal{H}\rightarrow \mathbb{R}^d$ such that:
\begin{align}\label{eq:kernel_op}
\mathcal{O}(p) = \tilde{\mathcal{O}}_{\psi}(T_{\psi}(p)) \;\;\; \forall p \in \mathcal{P}.
\end{align}
Informally, we can see that the operator $\tilde{\mathcal{O}}_{\psi}$ can be defined as the composition of $\mathcal{O}$ with the inverse of $T_{\psi}$.

\subsection{Connections with mean-field inference}

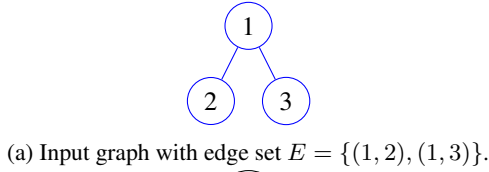
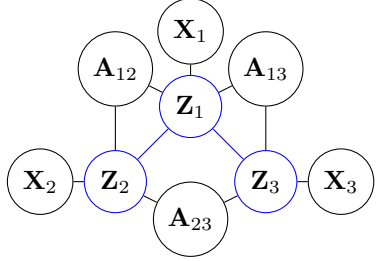
\begin{figure}[t]
\centering
\begin{subfigure}[b]{0.45\textwidth}
\centering
\begin{tikzpicture}
 \node[circle, blue, text=black] (a) at (0.5,1) [draw, minimum width=0.5cm,minimum height=0.5cm] {2};
 \node[circle, blue, text=black] (b) at (1.5,1) [draw, minimum width=0.5cm,minimum height=0.5cm] {3};
 \node[circle, blue, text=black] (c) at (1,2) [draw, minimum width=0.5cm,minimum height=0.5cm] {1};
 \foreach \from/\to in {a/c, c/b}
\draw [-, blue] (\from) -- (\to);
    \end{tikzpicture}
\caption{Input graph with edge set $E=\{(1,2), (1,3)\}$.}\label{fig:lvm_enc_graph}
\end{subfigure}
\begin{subfigure}[b]{0.45\textwidth}
\centering
\begin{tikzpicture}
\node[circle, black, text=black] (d) at (0,1) [draw, minimum width=0.15cm,minimum height=0.15cm]{$\mathbf{X}_2$};
 \node[circle, blue, text=black] (a) at (1,1) [draw, minimum width=0.15cm,minimum height=0.15cm] {$\mathbf{Z}_2$};
 \node[circle, black, text=black] (e) at (4,1) [draw, minimum width=0.15cm,minimum height=0.15cm]{$\mathbf{X}_3$};
 \node[circle, blue, text=black] (b) at (3,1) [draw, minimum width=0.15cm,minimum height=0.15cm] {$\mathbf{Z}_3$};
  \node[circle, black, text=black] (f) at (2,3) [draw, minimum width=0.05cm,minimum height=0.05cm]{$\mathbf{X}_1$};
 \node[circle, blue, text=black] (c) at (2,2) [draw, minimum width=0.05cm,minimum height=0.05cm] {$\mathbf{Z}_1$};
  \node[circle, black, text=black] (g) at (1,2.5) [draw, minimum width=0.35cm,minimum height=0.35cm]
  {$\mathbf{A}_{12}$};
  \node[circle, black, text=black] (h) at (2,0.5) [draw, minimum width=0.35cm,minimum height=0.35cm]{$\mathbf{A}_{23}$};
  \node[circle, black, text=black] (i) at (3,2.5) [draw, minimum width=0.35cm,minimum height=0.35cm]{$\mathbf{A}_{13}$};
 \foreach \from/\to in {a/c, c/b}
\draw [-, blue] (\from) -- (\to);
 \foreach \from/\to in {d/a, e/b, f/c}
\draw [-, black] (\from) -- (\to);
\foreach \from/\to in {g/a, g/c, h/a, h/b, i/b, i/c}
\draw [-, black] (\from) -- (\to);
    \end{tikzpicture}
\caption{Latent variable model $\mathcal{G}$ satisfying Property~\ref{thm:imap} with  $\mathbf{A}_{12}\neq 0,\mathbf{A}_{23}=0, \mathbf{A}_{13}\neq 0 $.}\label{fig:lvm_enc_model}
\end{subfigure}
\caption{Interpreting message passing in Graph Neural Networks via Kernel Embeddings and Mean-field inference}\label{fig:lvm_enc}
\end{figure}

Locality preference for representational learning is a key inductive bias for graphs. We formulate this using an (undirected) graphical model $\mathcal{G}$ over $\mathbf{X}$,  $\mathbf{A}$, and $\{\mathbf{Z}_1, \cdots, \mathbf{Z}_n\}$.
As in a GNN, we assume that $\mathbf{X}$ and $\mathbf{A}$ are observed and specify conditional independence structure in 
a
conditional distribution over the latent variables, denoted as $r(\mathbf{Z}_1, \cdots, \mathbf{Z}_n \vert \mathbf{A}, \mathbf{X})$. We are particularly interested in models that satisfy the following property.
\begin{property}\label{thm:imap}
The edge set $E$ defined by the adjacency matrix  $\mathbf{A}$ is an undirected I-map for the 
distribution $r(\mathbf{Z}_1, \cdots, \mathbf{Z}_n \vert \mathbf{A}, \mathbf{X})$.
\end{property}

In words, the above property implies that according to the conditional distribution over $\mathbf{Z}$, any individual  $\mathbf{Z}_i$ is independent of all other $\mathbf{Z}_j$ 
when conditioned on $\mathbf{A}$, $\mathbf{X}$, and the neighboring latent variables of node $i$ as determined by the edge set $E$. See Figure~\ref{fig:lvm_enc} for an illustration. 

A mean-field (MF) approximation for $\mathcal{G}$  approximates the conditional distribution  $r(\mathbf{Z}_1, \cdots, \mathbf{Z}_n \vert \mathbf{A}, \mathbf{X})$ as:
\begin{align}\label{eq:enc_gnn_mf}
r(\mathbf{Z}_1, \cdots, \mathbf{Z}_n \vert \mathbf{A}, \mathbf{X}) &\approx \prod_{i=1}^n q_{\phi_i}(\mathbf{Z}_i \vert \mathbf{A}, \mathbf{X})
\end{align}
where $\phi_i$ denotes the set of parameters for the $i$-th variational 
marginal.
These parameters are optimized by minimizing the KL-divergence between the variational 
and the true 
conditional distributions:
\begin{align}\label{eq:lvm_inf_enc}
\min_{\phi_1,\cdots,\phi_n} \mathrm{KL}\left(\prod_{i=1}^n q_{\phi_i}\left(\mathbf{Z}_i \vert \mathbf{A}, \mathbf{X}) \Vert r(\mathbf{Z}_1, \cdots, \mathbf{Z}_n \vert \mathbf{A}, \mathbf{X}\right)\right)
\end{align}

Using standard variational arguments~\citep{wainwright2008graphical}, we know that the optimal variational marginals  assume the following functional form:
\begin{align}
q_{\phi_i}(\mathbf{Z}_i \vert \mathbf{A}, \mathbf{X}) &= \mathcal{O}^{MF}_{\mathcal{G}}\left(\mathbf{Z}_i, \{q_{\phi_j}\}_{j \in \mathcal{N}(i)}\right)
\end{align}
where $\mathcal{N}(i)$ denotes the neighbors of $\mathbf{Z}_i$ in $\mathcal{G}$ 
and $\mathcal{O}$ is a function determined by the fixed point equations that depends on the potentials associated with $\mathcal{G}$. 
Importantly, the above functional form suggests that the optimal marginals in mean field inference are locally consistent that they are only a function of the neighboring marginals.
An iterative algorithm for mean-field inference is to perform message passing over the underlying graph until convergence. With an appropriate initialization at $l=0$, the updated marginals at iteration $l\in \mathbb{N}$ are given as:
\begin{align}\label{eq:marginal_mp}
q_{\phi_i}^{(l)}(\mathbf{Z}_i \vert \mathbf{A}, \mathbf{X}) &= \mathcal{O}^{MF}_{\mathcal{G}}\left(\mathbf{Z}_i, \{q_{\phi_j}^{(l-1)}\}_{j \in \mathcal{N}(i)}\right).
\end{align}

We will sidestep deriving $\mathcal{O}$, and instead use the kernel embeddings of the variational marginals to directly reason in the embedding space. 
That is, we assume we have an injective embedding for each marginal $q_{\phi_i}$ given by $\boldsymbol{\mu}_i = \mathbb{E}_{\mathbf{Z}_i\sim q_{\phi_i}}[\psi(\mathbf{Z}_i)]$ for some feature map $\psi: \mathbb{R}^k \rightarrow\mathbb{R}^{k'} $ 
and directly use the equivalence established in Eq.~\eqref{eq:kernel_op} iteratively. For mean-field inference via message passing as in Eq.~\eqref{eq:marginal_mp}, this gives us the following recursive expression for iteratively updating the embeddings at iteration $l\in \mathbb{N}$:
\begin{align}\label{eq:rec_mu}
\boldsymbol{\mu}^{(l)}_i = \tilde{O}^{MF}_{\psi,\mathcal{G}}\left(\{\boldsymbol{\mu}^{(l-1)}_j\}_{j \in \mathcal{N}(i)}\right)
\end{align}
with an appropriate base case for  $\boldsymbol{\mu}^{(0)}_i$. 
We then have the following result:

\begin{theorem}\label{thm:gnn_imap}
Let $\mathcal{G}$ be any undirected latent variable model such that the conditional distribution $r(\mathbf{Z}_1, \cdots, \mathbf{Z}_n \vert \mathbf{A}, \mathbf{X})$ expressed by the model satisfies Property~\ref{thm:imap}. 

Then there exists a choice of $\eta_l$, $\mathcal{F}_l$, 
$\mathbf{W}_l$, 
and $\mathbf{B}_l$ such that for all $\{\boldsymbol{\mu}^{(l-1)}_i\}_{i=1}^{n}$, the GNN propagation rule in Eq.~\eqref{eq:gnn_mp} is computationally equivalent to updating $\{\boldsymbol{\mu}^{(l-1)}_i\}_{i=1}^{n}$ via a first order approximation of Eq.~\eqref{eq:rec_mu}.  
\end{theorem}
\begin{proof}
See Appendix~\ref{app:enc}.
\end{proof}

While $\eta_l$ and $\mathcal{F}_l$ are typically fixed beforehand, the parameters $\mathbf{W}_l$, 
and $\mathbf{B}_l$ are directly learned from data in practice. Hence we have shown that a GNN is a good model for computation with respect to latent variable models that attempt to capture inductive biases relevant to graphs, \textit{i.e.}, ones where the latent feature vector for every node is conditionally independent from everything else given the feature vectors of its neighbors (and $\mathbf{A}$, $\mathbf{X}$). Note that such a graphical model would satisfy Property~\ref{thm:imap} but is in general different from the posterior specified by the one in Figure~\ref{fig:lvm}. However if the true (but unknown) posterior on the latent variables for the model proposed in Figure~\ref{fig:lvm} could be expressed as an equivalent model satisfying the desired property, then Theorem~\ref{thm:gnn_imap} indeed suggests the use of GNNs for parameterizing variational posteriors, as we do so in the case of \name{}.

\section{Discussion \& Related Work}\label{sec:related}

Our framework effectively marries \textit{probabilistic modeling} and  \textit{representation learning} on graphs. We review some of the dominant prior works in these fields below.

\paragraph{Probabilistic modeling of graphs.}
The earliest \textit{probabilistic models of graphs} proposed to generate graphs by creating an edge between any pair of nodes with a constant probability~\citep{erdos1959random}. Several alternatives have been proposed since; \textit{e.g.}, the small-world model generates graphs that exhibit local clustering~\citep{watts1998collective}, the Barabasi-Albert models preferential attachment wherein high-degree nodes are likely to form edges with newly added nodes~\citep{barabasi1999random}, the stochastic block model is based on inter and intra community linkages~\citep{holland1983stochastic} etc.
 We direct the interested reader to prominent surveys on this topic~\citep{newman2003structure,mitzenmacher2004brief,chakrabarti2006graph}.

\paragraph{Representation learning on graphs.}
For \textit{representation learning on graphs}, there are broadly three kinds of approaches: matrix factorization, random walk based approaches, and graph neural networks. We include a brief discussion on the first two kinds in Appendix~\ref{app:related} and refer the reader to \citet{hamilton2017representation} for a recent survey.

Graph neural networks, a collective term for networks that operate over graphs using message passing, have shown success on several downstream applications, e.g., ~\citep{duvenaud2015molecular,li2015gated,kearnes2016molecular,kipf2016semi,hamilton2017inductive} and the references therein. \citet{gilmer2017neural} provides a comprehensive characterization of these networks in the message passing setup. We used Graph Convolution Networks, partly to provide a direct comparison with GAE/VGAE and leave the exploration of other GNN variants for future work.

\paragraph{Latent variable models for graphs.} Hierarchical Bayesian models parameterized by deep neural networks have been recently proposed for graphs~\citep{hu2017deep,wang2017relational}. Besides being restricted to single graphs, these models are limited since inference requires running expensive Markov chains~\citep{hu2017deep} or are task-specific~\citep{wang2017relational}. \citet{johnson2017transitions} and \citet{kipf2018neural} generate graphs as latent representations learned directly from data. In contrast, we are interested in modeling observed (and not latent) relational structure. Finally, there has been a fair share of recent work for generation of special kinds of graphs, such as parsed trees of source code~\citep{maddison2014structured} and SMILES representations  for molecules~\citep{olivecrona2017denovo}. 

 Several deep generative models for graphs have recently been proposed. Amongst adversarial generation approaches, \citet{wang2017graphgan} and \citet{bojchevski2018netgan} model local graph neighborhoods and random walks on graphs respectively. \citet{li2018learning} and \citet{you2018graphrnn} model graphs as sequences and generate graphs via autoregressive procedures. Adversarial and autoregressive approaches are successful at generating graphs, but do not directly allow for inferring latent variables via encoders. Latent variable generative models have also been proposed for generating small molecular graphs~\citep{jin2018junction,samanta2018designing,simonovsky2018graphvae}. These methods involve an expensive decoding procedure that limits scaling to large graphs. Finally, closest to our framework is the GAE/VGAE approach \citep{kipf2016variational} discussed in Section~\ref{sec:exps}. \citet{pan2018adversarially} extends this approach with an adversarial regularization framework but retain the inner product decoder. Our work proposes a novel multi-step decoding mechanism based on graph refinement.

\section{Conclusion \& Future Work}
We proposed \name{}, a scalable deep generative model for graphs based on variational autoencoding. The encoders and decoders in \name{} are parameterized by graph neural networks that propagate information locally on a graph. 
Our proposed decoder performs a multi-layer iterative decoding comprising of alternate inner product operations and message passing on the intermediate graph.

Current generative models for graphs are not permutation-invariant and are learned by feeding graphs with a fixed or heuristic ordering of nodes.
This is an exciting challenge for future work, which could potentially be resolved by incorporate graph representations robust to permutation invariances~\citep{verma2017hunt} or modeling distributions over permutations of node orderings via recent approaches such as NeuralSort~\citep{grover2019stochastic}.
Extending \name{} for modeling richer graphical structure such as heterogeneous and time-varying graphs, as well as integrating domain knowledge within \name{} decoders for applications in generative design and synthesis \textit{e.g.}, molecules, programs, and parse trees is another interesting future direction. 

Finally, our theoretical results in Section~\ref{sec:interpret} suggest that a principled design of layerwise propagation rules in graph neural networks inspired by additional message passing inference schemes~\citep{dai2016discriminative,gilmer2017neural} is another avenue for future research.

\section*{Acknowledgements}
 This research has been supported by Siemens, a Future of Life Institute grant, NSF grants (\#1651565, \#1522054, \#1733686), ONR (N00014-19-1-2145), AFOSR (FA9550-19-1-0024), and an Amazon AWS Machine Learning Grant. AG is supported by a Microsoft Research Ph.D. fellowship and a Stanford Data Science Scholarship. We would like to thank Daniel Levy for helpful comments on early drafts.

\bibliography{refs}
\bibliographystyle{icml2019}

\clearpage
\appendix
\section*{Appendices}

\section{Proof of Theorem~\ref{thm:gnn_imap}}\label{app:enc}
\begin{proof}
For simplicity, we state the proof for a single variational marginal embedding $\mu^{(l)}_i$ and consider that $\mu^{(l)}_i$ for all $l \in \mathbb{N} \cup 0$ are unidimensional. 

Let us denote $\mathbf{N}_i^{(l)} \in \mathbb{R}^{n}$ to be the vector of neighboring kernel embeddings at iteration $l$ such that the $j$-th entry of $\mathbf{N}_i^{(l)}$ corresponds to $\mu^{(l)}_j$ if $j \in \mathcal{N}(i)$ and zero otherwise.  Hence, we can rewrite Eq.~\eqref{eq:rec_mu} as:
\begin{align}\label{eq:rec_mu_n}
\mu^{(l)}_i =\tilde{O}_{\psi, \mathcal{G}}\left(\mathbf{N}_i^{(l-1)}\right)
\end{align}
where we have overloaded $\tilde{O}_{\psi, , \mathcal{G}}$ to now denote a function that takes as argument an $n$-dimensional vector of marginal embeddings.

Assuming that the function $\tilde{O}_{\psi, \mathcal{G}}$ is differentiable, a first-order Taylor expansion of Eq.~\eqref{eq:rec_mu_n} around the origin $\mathbf{0}$ is given by:
\begin{align}\label{eq:taylor}
\mu^{(l)}_i \approx \tilde{O}_{\psi, \mathcal{G}}\left(\mathbf{0}\right)
+  \mathbf{N}_i^{(l-1)} \cdot \nabla \tilde{O}_{\psi, \mathcal{G}}\left(\mathbf{0}\right). 
\end{align}

Since every marginal density is unidimensional, we now consider a GNN with a single activation per node in every layer, $i.e.$, $\mathbf{H}^{(l)}\in \mathbb{R}^{n}$  for all $l \in \mathbb{N} \cup 0$. 
 This also implies that the bias can be expressed as an  $n$-dimensional vector, \textit{i.e.}, $\mathbf{B}_{l} \in \mathbb{R}^{n}$ and we have a single weight parameter $W_l \in \mathbb{R}$.
For a single entry of $\mathbf{H}^{(l)}$, we can specify Eq.~\eqref{eq:gnn_mp} component-wise as:
\begin{align}\label{eq:gnn_mp_comp}
H^{(l)}_i = \eta_{l}\left( B_{l,i} + \sum_{f \in \mathcal{F}_{l}}f(\mathbf{A}_i)\mathbf{H}^{(l-1)}W_l 
\right)
\end{align}
where $\mathbf{A}_i$ denotes the $i$-th row of $\mathbf{A}$ and is non-zero only for entries corresponding to the neighbors of node $i$.

Now, consider the following instantiation of Eq.~\eqref{eq:gnn_mp_comp}:
\begin{itemize}
\item $\eta_{l}\leftarrow \mathcal{I}$ (identity function)
\item $B_{l,i} \leftarrow  \tilde{O}_{\psi, \mathcal{G}}\left(\mathbf{0}\right)$
\item A family of $n$ transformations $\mathcal{F}_l = \{f_{l,j}\}_{j=1}^n$ where $f_{l,j}(\mathbf{A}_i) = 
 \frac{\partial \tilde{O}_{\psi, \mathcal{G}}}{\partial \mu^{(l-1)}_j}\left(\mathbf{0}\right)
A_{ij}$
\item $H^{(l-1)}_i \leftarrow \mu^{(l-1)}_i$
\item $W_{l} \leftarrow 1$.
\end{itemize}

With the above substitutions, we can equate the first order approximation in Eq.~\eqref{eq:rec_mu_n} to the GNN message passing rule in Eq.~\eqref{eq:gnn_mp_comp}, thus completing the proof. With vectorized notation and use of matrix calculus in Eqs.~(\ref{eq:rec_mu_n}-\ref{eq:gnn_mp_comp}), the derivation above also applies to entire vectors of variational marginal embeddings with arbitrary dimensions.
\end{proof}

\section{Experiment Specifications}\label{app:expt}

\subsection{Link prediction}

We used the SC implementation from \citep{pedregosa2011scikit} and public implementations for others made available by the authors. For SC, we used a dimension size of $128$. For DeepWalk and node2vec which uses a skipgram like objective on random walks from the graph, we used the same dimension size and default settings used in \citep{perozzi2014deepwalk} and \citep{grover2016node2vec} respectively of $10$ random walks of length $80$ per node and a context size of $10$. For node2vec, we searched over the random walk bias parameters using a grid search in $\{0.25, 0.5, 1, 2, 4\}$ as prescribed in the original work. For GAE and VGAE, we used the same architecture as VGAE and Adam optimizer with learning rate of $0.01$. 

For Graphite-AE and Graphite-VAE, we used an architecture of 32-32 units for the encoder and 16-32-16 units for the decoder (two rounds of iterative decoding before a final inner product). The model is trained using the Adam optimizer~\citep{kingma2013auto} with a learning rate of $0.01$. All activations were RELUs.The dropout rate (for edges) and $\lambda$ were tuned as hyperparameters on the validation set to optimize the AUC, whereas traditional dropout was set to 0 for all datasets. Additionally, we trained every model for $500$ iterations and used the model checkpoint with the best validation loss for testing.  Scores are reported as an average of 50 runs with different train/validation/test splits (with the requirement that the training graph necessarily be connected).

For \name{}, we observed that using a form of skip connections to define a linear combination of the initial embedding $\mathbf{Z}$ and the final embedding $\mathbf{Z}^*$ is particularly useful. The skip connection consists of a tunable hyperparameter $\lambda$ controlling the relative weights of the embeddings.
The final embedding of \name{} is a function of the initial embedding $Z$ and the last induced embedding $Z^*$.  We consider two functions to aggregate them into a final embedding.  That is, $(1-\lambda)Z + \lambda Z^*$ and $Z + \lambda Z^* / \left\lVert Z^* \right\rVert$, which correspond to a convex combination of two embeddings, and an incremental update to the initial embedding in a given direction, respectively. Note that in either case, GAE and VGAE reduce to a special case of \name{}, using only a single inner-product decoder (\textit{i.e.}, $\lambda=0$).  On Cora and Pubmed final embeddings were derived through convex combination, on Citeseer through incremental update.

\paragraph{Scalability.} We experimented with learning VGAE and \name{} models by subsampling $\vert E \vert$ random entries for Monte Carlo evaluation of the objective at each iteration. The corresponding AUC scores are shown in Table~\ref{table-scale}. The results suggest that \name{} can effectively scale to large graphs without significant loss in accuracy. 
The AUC results trained with edge subsampling as we vary the subsampling coefficient $c$ are shown in Figure~\ref{fig:scalability}.

\begin{table}[t]
\centering
  \caption{AUC scores for link prediction with Monte Carlo subsampling during training. Higher is better.}
  \label{table-scale}
  \vspace{0.05in}
  \centering
  \begin{tabular}{|c|c|c|c|}
    \toprule
	& Cora & Citeseer & Pubmed \\
    \midrule
    VGAE & 89.6 & 92.2 & 92.3\\
    Graphite & \textbf{90.5} & \textbf{92.5} & \textbf{93.1}\\
    \bottomrule
  \end{tabular}
\end{table}

\begin{figure}[t]
\centering
\includegraphics[width=0.5\textwidth]{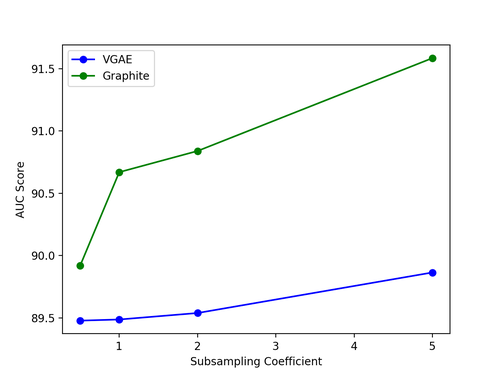}
\caption{AUC score of VGAE and Graphite with subsampled edges on the Cora dataset.}\label{fig:scalability} 
\end{figure}

\subsection{Semi-supervised node classification}

We report the baseline results for  SemiEmb~\citep{weston2008embedding},  DeepWalk~\citep{perozzi2014deepwalk}, ICA~\citep{lu2003classification} and  Planetoid~\citep{yang2016embeddings} as specified in~\citep{kipf2016semi}.  GCN uses a 32-16 architecture with ReLu activations and early stopping after $10$ epochs without increasing validation accuracy.  The \name{} model uses the same architecture as in link prediction (with no edge dropout). The parameters of the posterior distributions are concatenated with node features to predict the final output.  The parameters are learned using the Adam optimizer~\citep{kingma2013auto} with a learning rate of $0.01$. All accuracies are taken as an average of 100 runs.

\subsection{Density estimation}
To accommodate for input graphs of different sizes, we learn a model architecture specified for the maximum possible nodes (\textit{i.e.}, $20$ in this case). While feeding in smaller graphs, we simply add dummy nodes disconnected from the rest of the graph. The dummy nodes have no influence on the gradient updates for the parameters affecting the latent or observed variables involving nodes in the true graph.
For the experiments on density estimation, we pick a graph family, then train and validate on graphs sampled exclusively from that family.  We consider graphs with nodes ranging between 10 and 20 nodes belonging to the following graph families :
\begin{itemize}
\item Erdos-Renyi~\citep{erdos1959random}: each edge independently sampled with probability $p = 0.5$
\item Ego Network: a random Erdos-Renyi graph with all nodes neighbors of one randomly chosen node
\item Random Regular: uniformly random regular graph with degree $d = 4$
\item Random Geometric: graph induced by uniformly random points in unit square with edges between points at euclidean distance less than $r = 0.5$
\item Random Power Tree: Tree generated by randomly swapping elements from a degree distribution to satisfy a power law distribution for $\gamma = 3$
\item Barabasi-Albert ~\citep{barabasi1999random}: Preferential attachment graph generation with attachment edge count $m = 4$
\end{itemize}

We use convex combinations over three successively induced embeddings.  Scores are reported over an average of 50 runs.  Additionally, a two-layer neural net is applied to the initially sampled embedding $\mathbf{Z}$ before being fed to the inner product decoder for GAE and VGAE, or being fed to the iterations of Eqs.~\eqref{eq:int_graph_decoding} and \eqref{eq:refine_graph_dec} for both Graphite-AE and Graphite-VAE.

  \section{Additional Related Work}\label{app:related}
  \textbf{Factorization based approaches}, such as Laplacian Eigenmaps~\citep{belkin2002laplacian} and IsoMaps~\citep{saxena2004non}, operate on a matrix representation of the graph, such as the adjacency matrix or the graph Laplacian. These approaches are closely related to dimensionality reduction and can be computationally expensive for large graphs. 

\textbf{Random-walk methods} are based on variations of the skip-gram objective~\citep{mikolov2013distributed} and learn representations by linearizing the graph through random walks. These methods, in particular DeepWalk~\citep{perozzi2014deepwalk}, LINE~\citep{tang2015line}, and node2vec~\citep{grover2016node2vec}, learn general-purpose unsupervised representations that have been shown to give excellent performance for  semi-supervised node classification and link prediction. Planetoid~\citep{yang2016embeddings} learn representations based on a similar objective specifically for semi-supervised node classification by explicitly accounting for the available label information during learning.

\end{document}